\documentclass[conference]{IEEEtran}
	\IEEEoverridecommandlockouts
	\usepackage{cite}
	\usepackage{amsmath,amssymb,amsfonts}
        \usepackage{tikz}
        \usetikzlibrary{patterns.meta}

        \usepackage{pgfplots}
        \pgfplotsset{compat=1.18}

	\usepackage{graphicx}
	\usepackage{textcomp}
        \usepackage{breqn}
	\usepackage{amsthm}
	\newtheorem{theorem}{Theorem}
	\newtheorem{lemma}{Lemma}
        \newtheorem{proposition}{Proposition}
	
	\usepackage{graphicx}
        \usepackage{subcaption}
	\theoremstyle{definition}

	\theoremstyle{remark}

         \usepackage{algorithm}
        \usepackage{algpseudocode}
        
        \newtheorem{assumption}{Assumption}
        
\newcommand{\BEAS}{\begin{eqnarray*}}
	\newcommand{\EEAS}{\end{eqnarray*}}
\newcommand{\BEA}{\begin{eqnarray}}
	\newcommand{\EEA}{\end{eqnarray}}
\newcommand{\BEQ}{\begin{equation}}
	\newcommand{\EEQ}{\end{equation}}
\newcommand{\BIT}{\begin{itemize}}
	\newcommand{\EIT}{\end{itemize}}

\newcommand{\ie}{{\it i.e.}}

\newcommand{\ones}{\mathbf 1}
\newcommand{\Rank}{\mathop{\bf rank}}

\newcommand{\Prob}{\mathop{\mathbb{P}}}
\newcommand{\Expect}{\mathop{\mathbb{E}}}
\newcommand{\Co}{\mathop {\bf conv}} % convex hull

\newcounter{exno}
\newlength{\exlabelwidth}

{\clearpage\section*{Exercises}%
	\addcontentsline{toc}{section}{Exercises}%
	\markright{Exercises}%
	\begin{small}\addtolength{\baselineskip}{-1.0pt}%
		\begin{list}{\bfseries\textsf{\thechapter.\arabic{exno}}}%
			{\usecounter{exno}%
				\settowidth{\exlabelwidth}{\bfseries\textsf{\thechapter.99}}%
				\setlength{\labelwidth}{\exlabelwidth}%
				\setlength{\labelsep}{.5em}%
				\setlength{\leftmargin}{0mm}%
				\setlength{\rightmargin}{0mm}}} %
		{\end{list}\clearpage%
	\end{small}\addtolength{\baselineskip}{1.0pt}}   % add markboth

\makeatletter
\long\def\@makecaption#1#2{
	\vskip 9pt 
	\begin{small}
		\setbox\@tempboxa\hbox{{\sffamily\bfseries #1} #2}
		\ifdim \wd\@tempboxa > 0.85\textwidth
		\begin{center}
			\begin{minipage}[t]{0.85\textwidth}
				\addtolength{\baselineskip}{-0.95pt}
				{\sffamily\bfseries #1} #2 \par
				\addtolength{\baselineskip}{0.95pt}
			\end{minipage}
		\end{center}
		\else 
		\hbox to\hsize{\hfil\box\@tempboxa\hfil}  
		\fi
	\end{small}\par
}
\makeatother

\newcounter{oursection}

\newcounter{lecture}

\newcommand{\calC}{{\mathcal C}}
\newcommand{\calD}{{\mathcal D}}
\newcommand{\calF}{{\mathcal F}}

\newcommand{\calR}{{\mathcal R}}

\newcommand{\N}{{\mathbb{N}}}
\newcommand{\reals}{{\mathbb{R}}}

\newcommand{\thetahat}{{\widehat{\theta}}}

\newcommand{\muhat}{{\widehat{\mu}}}

\newcommand{\st}{\mathrm{subject\: to\:}}
\DeclareMathOperator*{\maximize}{maximize}
\newcommand{\minimize}{\mathrm{minimize}\:}
	
	\usepackage{xcolor}
	\def\BibTeX{{\rm B\kern-.05em{\sc i\kern-.025em b}\kern-.08em
	    T\kern-.1667em\lower.7ex\hbox{E}\kern-.125emX}}
         \definecolor{DarkGreen}{RGB}{0,100,0}

\newcounter{edremcounter}
\usepackage[hidelinks]{hyperref}
\definecolor{darkred}{RGB}{150,0,0}
\definecolor{darkgreen}{RGB}{0,150,0}
\definecolor{darkblue}{RGB}{0,0,150}
\hypersetup{colorlinks=true, linkcolor=red, citecolor=blue, urlcolor=darkblue}

\begin{document}
		
		\title{Convex Methods for Constrained Linear Bandits
		}

            % \author{Amirhossein Afsharrad, }
            \author{
    Amirhossein Afsharrad$^1$,
    Ahmadreza Moradipari$^2$,
    Sanjay Lall$^1$
    \thanks{$^1$ Department of Electrical Engineering, Stanford University, 
    \texttt{\{afsharrad,lall\}@stanford.edu}.
    $^2$ Department of Electrical and Computer Engineering, University of California Santa Barbara, \texttt{ahmadreza$\_$moradipari@ucsb.edu}.
    This work was supported by NSF under ECCS CPS project number 2125511.
    }
}

		% \author{\IEEEauthorblockN{Amirhossein Afsharrad}
		% 	\IEEEauthorblockA{\textit{Department of Electrical Engineering} \\
		% 		\textit{Stanford University}\\
		% 		Stanford, CA \\
		% 		afsharrad@stanford.edu}
		% 	\and
  %               \IEEEauthorblockN{Ahmadreza Moradipari}
		% 	\IEEEauthorblockA{\textit{  Department of Electrical and Computer Engineering} \\
		% 		\textit{University of California, Santa Barbara}\\
		% 		Santa Barbara, CA \\
		% 		ahmadreza$\_$moradipari@ucsb.edu}
  %               \and
		% 	\IEEEauthorblockN{Sanjay Lall}
		% 	\IEEEauthorblockA{\textit{Department of Electrical Engineering} \\
		% 		\textit{Stanford University}\\
		% 		Stanford, CA \\
		% 		lall@stanford.edu}
  		% }
		
		\maketitle

\begin{abstract}

% In this work, we study the algorithmic aspects of linear bandits with linear constraints, specifically emphasizing implementation details which have not received much attention in the existing literature. While optimal regret bounds have been established for various algorithms, practical implementation can present challenges. 
% We identify and address these challenges in this work.

Recently, bandit optimization has received significant attention in real-world safety-critical systems that involve repeated interactions with humans. 
While there exist various algorithms with performance guarantees in the literature, practical implementation of the algorithms has not received as much attention.  
This work presents a comprehensive study on the computational aspects of  safe bandit algorithms, specifically safe linear bandits, by introducing a framework that leverages convex programming tools to create computationally efficient policies. In particular, 
we first characterize the properties of the optimal policy for safe linear bandit problem and then propose an end-to-end pipeline of safe linear bandit algorithms that only involves solving convex problems. We also numerically evaluate the performance of our proposed methods.

% Addressing the challenges of non-convex optimization and optimization over complex decision sets, the paper proposes an innovative pipeline of algorithms that are restricted to convex problem solving, thus ensuring efficiency and practicality. 

% By focusing on decision sets that are unions of convex sets described by convex inequalities, the research circumvents the difficulties associated with non-convex sets. The paper details the characteristics of optimal policies, presents a method for their computation, and introduces a general algorithm with improved performance guarantees. It also proposes a problem-dependent approach that optimizes performance bounds for specific problem classes. Empirical experiments substantiate the effectiveness of the proposed methods, signaling a significant advancement in the field of safe linear bandits.

\end{abstract}

\section{Introduction}

Recently, bandit optimization has received significant attention in real-world cyber-physical systems that involve repeated interactions with humans. In such cases, a learner repeatedly interacts with an unknown environment. During each interaction, it selects an action from a given action set and observes its corresponding reward. The learner's goal is to maximize the accumulated reward. However, these systems are bound by safety constraints that must be respected during these interactions. Consequently, traditional bandit algorithms may not be directly applicable in these contexts. Indeed, proper and nontrivial modifications are necessary to enable the use of bandit algorithms in safety-critical systems. To achieve this, new research directions have emerged, focusing on designing constraint bandit algorithms with provable guarantees. In these settings, the environment is subject to a set of unknown operational constraints. Depending on the nature of these constraints, various constrained stochastic bandit settings have been formulated and analyzed. In our work, we concentrate on the linear stochastic bandit problem that is constrained by a set of unknown linear constraints.

A linear bandit (LB) is a variant of the multi-armed bandit (MAB) problem
in which each action is associated with a feature vector $x$ and the expected reward of playing each action is equal to the inner product of its feature vector and an unknown parameter vector $\theta_*$. Two efficient approaches have been developed for LB: linear UCB (LUCB) \cite{Dani08stochasticlinear,Tsitsiklis,abbasi2011improved} and linear Thompson sampling \cite{agrawal2013thompson, abeille2017linear}.  A diverse body of related works on linear stochastic bandits has considered the effect of safety constraints that need to be respected during all the rounds of the algorithm. 
An algorithm is called stage-wise safe if the safety constraint is not
violated with high probability over all rounds. Such algorithms have
been proposed for for linear UCB \cite{amani2019linear} and for linear
Thompson sampling~\cite{moradipari2021safe, moradipari2020linear}.
In the more relaxed setting, where the algorithm is allowed to violate the safety constraint for some limited rounds, \cite{chen2022strategies} has proposed safe algorithms with a provable upper bound on the total number constraint violations. 
Our setting is inspired by the work of \cite{pacchiano2021stochastic}, where the agent's objective is to produce a series of policies that yield the highest expected cumulative reward, all the while maintaining that the expected cost of the policy constructed in each round stays below a specified threshold.

In this work, we investigate the computational aspects of safe linear bandit algorithms. Various methods have been developed as shown in~\cite{pacchiano2021stochastic, moradipari2021safe, varma2023stochastic, khezeli2020safe, hutchinson2023impact,chen2022doubly}, which produce policies with precise performance guarantees. In this paper, we utilize convex programming tools to build a framework using these algorithms, allowing for explicit computation of policies. We aim to address two main challenges. First, standard methods require solving a non-convex optimization problem at each time step of the bandit algorithms. This poses a computational challenge, as finding a globally efficient solution for this class of problems can become NP-hard in certain cases, as noted in~\cite{Dani2008StochasticLO}. Second, standard algorithms necessitate optimization over a set of probability distributions. While straightforward for convex decision sets, the complexity is dependent on the form of the decision set and can pose challenges for some non-convex decision sets. Our primary contribution is an end-to-end pipeline of algorithms for constrained bandits with performance guarantees, which only involve solving convex optimization problems. This ensures computational efficiency as all the algorithms can be efficiently implemented using only a convex solver. In order to address the second aforementioned challenge, we focus on decision sets that are a union of convex sets, each described by convex inequalities.

The rest of the paper is organized as follows: Section \ref{sec:prelim} presents some preliminary material. In Section \ref{sec:prob_formulation} we state the formal version of the problem we are addressing. In Section \ref{sec:opt_policy}, we provide characteristics of an optimal policy, offering insight into what one might expect from such a policy, and propose a method to compute such a policy. Section \ref{sec:l1_oplb} introduces a general computationally efficient algorithm with performance guarantees to address the constrained bandit problem. In Section \ref{sec:linear_relaxation} we propose a novel problem-dependent approach that improves the performance bound of the previous section and can achieve optimal performance for specific classes of problems. Section \ref{sec:experiments} presents experiments that illustrate the performance of our methods.

\section{Preliminaries}\label{sec:prelim}

Before delving into the main problem formulation and our results, we introduce a set of definitions and lemmas in this section. 

\textbf{Norms.} For a vector \(x \in \mathbb{R}^d\) and
positive definite matrix $\Sigma \in \reals^{d\times d}$ we define
    \[\|x\|_{\Sigma,p} = \|\Sigma^{1/2} x\|_p,\]
In     particular, for \(p = 2\), we have
    \[\|x\|_{\Sigma,2} = \sqrt{x^\top \Sigma x}.\]

    \begin{lemma}[Caratheodory’s theorem]\label{lemma:caratheodory}
        Every point in the convex hull of a set \( S\subset\reals^d \)
        can be expressed as a convex combination of at most \( d+1 \) points from \( S \).
    \end{lemma}

    \begin{lemma}[Linear program basic feasible solution]\label{lemma:bfs}
        The linear program 
        \begin{align*}
        \begin{split}
        \maximize\quad& c^\top x \\
        \st \quad& Ax=b\\
        & x\geq 0 \\
    \end{split}
    \end{align*}
    has a solution with at most \( p \) non-zero entries, where \( A\in\reals^{p\times q} \) is a fat full-rank matrix. This solution is called a basic feasible solution.
    \end{lemma}

    \begin{lemma}[Convex hull of the union of convex sets \cite{boyd_vandenberghe_2004}]\label{lemma:union_of_convex}
    Consider the problem
    \begin{align}\label{eq:conv_hull_prob_original}
    \begin{split}
        \minimize \quad &f_0(z) \\
        \st \quad& \Co\left(\bigcup_{i=1}^k \calD^i\right)\\
    \end{split}
    \end{align}
    where 
    \[\calD^i=\{x:f_{ij}(x)\leq 0, \: j=1,\cdots, k_i\}\]
    and each $f_{ij}:\reals^d\to\reals$ is convex.
    
    An approach to solving this problem is to solve the convex program
    \begin{align}\label{eq:conv_hull_prob}
    \begin{split}
        \minimize \quad &f_0(z) \\
        \st \quad& \alpha_i f_{ij}(x_i/\alpha_i)\leq 0, \quad i\in[k], j\in[k_i]\\
        & \ones^\top \alpha = 1 \\
        & \alpha \geq 0 \\
        & z = x_1 + \cdots + x_k
    \end{split}
    \end{align}
    over the variables $z, x_1, \cdots, x_k \in \reals^d$ and $\alpha_1,\cdots, \alpha_k\in \reals$. If \((z^\star, x_1^\star, \cdots, x_k^\star, \alpha_1^\star, \cdots, \alpha_k^*)\) is an optimal solution of \eqref{eq:conv_hull_prob}, then \(z^*\) is an optimal solution of \eqref{eq:conv_hull_prob_original}.
\end{lemma}

\section{Problem Formulation}\label{sec:prob_formulation}
\textbf{Initial setup.} We consider the linear bandit with linear constraints characterized by the reward parameter \(\theta_*\in\reals^d\) and the cost parameter \(\Gamma_*\in\reals^{m\times d}\). In each round \(t\), the agent is given a decision set \(\calD_t\subset \reals^d\) from which it has to choose an action \(x_t\). We assume that \( \calD_t \) is the union of \(n_t\) convex sets \(\calD_t^1,\cdots, \calD_t^{n_t}\), each of which being described via convex inequalities, \ie, 
\begin{align}\label{eq:union_of_convex}
    \begin{split}
        \calD_t^i&=\left\{x:f_t^{ij}(x)\leq 0, \: j=1,\cdots, k_t^i\right\}\\
        \calD_t&=\bigcup_{i=1}^{n_t}\calD_t^i.
    \end{split}
\end{align}
Upon taking action \(x_t\in \calD_t\), the agent observes a reward signal \(r_t=\theta_*^\top x_t+\eta_t^r\) and a cost signal vector \(c_t=\Gamma_* x_t + \eta_t^c\), where \(\eta_t^r\in \reals\) and \(\eta_t^c\in\reals^m\) are random variables of reward and cost noise, satisfying conditions that will be specified later. The agent selects its action \(x_t\in\calD_t\) in each round \(t\) according to its policy \(\pi_t\in\Delta_{\calD_t}\) at that round, \ie, \(x_t \sim \pi_t\).

\textbf{Objective.} The objective of the agent is to generate a sequence of policies \(\{\pi_t\}_{t=1}^T\) maximizing the \textit{expected cumulative reward} over \(T\) rounds. This should be achieved while satisfying the \textit{linear constraints}
\begin{equation}\label{eq01}
    \mathbb{E}_{x\sim \pi_t}\left(\Gamma_* x\right)\leq \tau, \quad \forall t \in [T],
\end{equation}
where the $i$th row of \(\Gamma_*\) is represented by \(\mu_{*i}\). The vector \(\tau \in \mathbb{R}^m\) is termed the \textit{constraint threshold vector} and is known to the agent. Additionally, the vector inequality in \eqref{eq01} is interpreted element-wise.

Consequently, the policy \(\pi_t\) that the agent chooses in each round \(t \in [T]\) must reside within the set of feasible policies defined over the action set \(\calD_t\), \ie,
\begin{equation}
\Pi_t = \left\{ \pi \in \Delta_{\calD_t} : \Expect_{x \sim \pi} \left( \Gamma_* x \right) \leq \tau \right\}.
\end{equation}
Optimizing for the maximum expected cumulative reward over \(T\) rounds can be rephrased as minimizing the constrained pseudo-regret across \(T\) rounds
\begin{equation}
    \calR_{\Pi}(\theta_*, T) = \sum_{t=1}^T \Expect_{x \sim \pi^*_t} \left( \theta_*^\top x \right) - \Expect_{x \sim \pi_t} \left( \theta_*^\top x \right),
\end{equation}
where \(\pi_t, \pi^*_t \in \Pi_t\) for all \(t \in [T]\). Here, \(\pi^*_t\) signifies the \textit{optimal feasible policy} during round \(t\), defined as
\begin{equation}
\pi^*_t = \max_{\pi \in \Pi_t} \Expect_{x \sim \pi_t} \left[ \theta_*^\top x \right].
\end{equation}
It is worth emphasizing that \(\pi^*_t\) refers to the optimal \textit{omniscient} feasible policy, one that is achievable by an agent that is informed of the hidden parameters \(\theta_*\) and \(\Gamma_*\). This should be distinctly recognized from the best achievable policy by an agent observing only noisy rewards and costs.

\textbf{Assumptions.} We operate under the following assumptions in our setting, which are standard in the linear bandit literature.
% Moreover, a significant portion of the existing results in this field is based on these assumptions.

\begin{assumption}
    The constraint parameter matrix \(\Gamma_*\in\reals^{m\times d}\) is fat and full-rank, \ie, \(m<d\) and \(\Rank(\Gamma_*)=m\).
\end{assumption}

\begin{assumption}
    For all \(t\in T\), the reward and cost noise random variables \(\eta_t^r\), \(\eta_t^c\)
    are conditionally \(R\)-sub-Gaussian, \ie,
    \begin{align*}
        &\Expect\left[\eta_t^r|\calF_{t-1}\right]=0, \quad \Expect\left[\exp\left(\alpha\eta_t^r\right)|\calF_{t-1}\right]\leq \exp\left(\alpha^2R^2/2\right), \\
        &\Expect\left[\eta_{t,i}^c|\calF_{t-1}\right]=0, \quad \Expect\left[\exp\left(\alpha\eta_{t,i}^c\right)|\calF_{t-1}\right]\leq \exp\left(\alpha^2R^2/2\right)
    \end{align*}
    for any \(\alpha\in\reals, i\in[m]\), where \(\calF_t\) is the filtration that includes all events \((x_{1:t+1}, \eta^r_{1:t}, \eta^c_{1:t})\) until the end of round \(t\).
\end{assumption}

\begin{assumption}
    There is a known constant \(S > 0\), such that \(\|\theta_*\|\leq S\) and \(\|\mu_{i*}\|\leq S^2\) for all \(i\in[m]\).
\end{assumption}

\begin{assumption}
    The decision set $\calD_t$ is bounded. Specifically,
    \(\max_{t\in[T]}\max_{x\in\calD_t}\|x\|\leq L\).
\end{assumption}

\begin{assumption}
    For all \(t\in [T]\) and \(x\in\calD_t\), the mean rewards and costs are bounded, \ie, \(\theta_*^\top x \in [0,1]\) and \(\mu_{i*}^\top x \in [0,1]\) for \(i\in[m]\).
\end{assumption}

\begin{assumption}
    There exists a universally safe action \(x_0 \in \calD_t\) for all \(t \in [T]\) associated with the cost vector \(c_0 \in \reals^m\). This means that \(\Gamma_* x_0 = c_0 < \tau\). For the sake of clarity, we assume that \(c_0 = 0\) and that its value is known. Extending this to the cases where \(c_0 \neq 0\) is known, or \(c_0\) is unknown, is straightforward. For further details on these scenarios, one can refer to \cite{pmlr-v130-pacchiano21a}.
\end{assumption}

\textbf{Summary.} To summarize, the problem data includes the reward vector \(\theta_*\), the constraint matrix \(\Gamma_*\), the constraint threshold vector \(\tau\), the problem horizon \(T\), the observation noise sub-Gaussian parameter \(R\), the reward and cost upper bound parameter \(S\), the known safe action \(x_0\), and the decision sets \(\mathcal{D}_1, \ldots, \mathcal{D}_T\), where each \(\mathcal{D}_t\) is characterized by a set of integers \(n_t, k_t^1, \ldots, k_t^{n_t}\) and a set of convex functions \(f_t^{i,j}\) with \(i \in [n_t], j \in [k_t^i]\).

Note that we are working within the specified class of decision sets, i.e., sets in the form of a union of convex sets each described by convex inequalities, exclusively for computational purposes. Nevertheless, it is important to highlight that our theoretical results and theorems remain valid for any arbitrary choice of decision sets.

\section{Main Results}
\subsection{The optimal feasible policy}\label{sec:opt_policy}
At each time step \( t \), the optimal feasible policy \( \pi^*_t \) is obtained by solving the following optimization problem:
\begin{align}\label{eq03}
    \begin{split}
        \maximize_{\pi\in\Delta_{\calD_t}}\quad &\Expect_{x\sim\pi}\left(\theta_*^\top x\right) \\
        \st \quad& \Expect_{x\sim\pi}\left(\Gamma_* x\right) \leq \tau
    \end{split}
\end{align}
While the reward and cost parameters \( \theta_* \) and \( \Gamma_* \) are unknown in the bandit setting, it is valuable to understand the structure of the optimal feasible policy 
\( \pi^*_t \) even when these parameters are known. Specifically, the optimization in \eqref{eq03} considers probability distributions over the decision set \( \calD_t \), and since \( \calD_t \) can be any arbitrary set, characterizing the optimal feasible policy can be a complex task. The subsequent theorem, an extension of Lemma 5 in \cite{pmlr-v130-pacchiano21a}, provides a characterization of the optimal feasible policy \( \pi^*_t \).

\begin{theorem}\label{thm_opt_policy}
    There exists an optimal feasible policy \( \pi^*_t \) that solves \eqref{eq03} with finite support of at most \( m+1 \) elements.
\end{theorem}
\begin{proof}
    First, observe that while \eqref{eq03} is an optimization over all choices of distributions \( \pi\in\calD_t \), the only component of \( \pi \) that plays a role in the optimization is \( \Expect_{x\sim\pi}(x) \). Thus, letting \( z=\Expect_{x\sim\pi}(x) \), solving \eqref{eq03} is equivalent to first solving 
    \begin{align}\label{eq04}
    \begin{split}
        \maximize_{z}\quad &\theta^\top z \\
        \st \quad& \Gamma z \leq \tau\\
        & z\in \Co(\calD_t)
    \end{split}
    \end{align}
    to find a solution \( z^* \), and then find a distribution \( \pi^*_t\in\Delta_{\calD_t} \) such that \( \Expect_{x\sim\pi^*_t}(x)=z^* \). Note that the constraint \( z\in \Co(\calD_t) \) has to be included in the new optimization problem since if \( z^*\notin \Co(\calD_t) \), then there is no distribution \( \pi\in\Delta_{\calD_t} \) whose expected value is \( z^* \).

    Now, let \( z^* \) be the solution of \eqref{eq04}. Since \( z \in \Co(\calD_t) \), we know that \( z \) is given by a convex combination of a finite number of elements in \( \calD_t \). Moreover, according to Caratheodory's theorem presented in Lemma \ref{lemma:caratheodory}, one such convex combination exists with at most \(d+1\) points. Thus, a set of points \( z_1, \cdots, z_{d+1}\in\calD_t \) and a set of non-negative scalars \( \alpha_1,\cdots, \alpha_{d+1} \) exist such that \( z^*=\sum_{i=1}^{d+1} \alpha_i z_i=Z\alpha \) and \( \sum_{i=1}^{d+1} \alpha_i =1 \), where \( Z\in\reals^{d\times (d+1)} \) is a matrix whose \( i \)th column is \( z_i \) and \( \alpha\in\reals^{d+1} \) is a vector whose \( i \) entry is \( \alpha_i \). Next, we form the following optimization problem:
    \begin{align}\label{eq:opt_policy_LP}
    \begin{split}
        \maximize_{\beta\in\reals^{d+1}}\quad &\theta^\top Z\beta \\
        \st \quad& \Gamma Z\beta \leq \tau\\
        & \ones^\top \beta = 1 \\
        & \beta \geq 0
    \end{split}
    \end{align}
    Note that if \( \beta \) is a solution of \eqref{eq:opt_policy_LP}, then \( z_{\beta}=Z\beta \) is a solution of \eqref{eq04}. The final step would be to show that a specific solution \( \beta^* \) for \eqref{eq:opt_policy_LP} exists with at most \( m+1 \) non-zero entries. This step is taken via Lemma \ref{lemma:bfs}, according to which \eqref{eq:opt_policy_LP} has a basic feasible solution that has no more than \(m+1\) non-zero elements. Note that \eqref{eq:opt_policy_LP} can be converted to the form given by Lemma \ref{lemma:bfs} by adding slack variables. Now, letting \( \beta^* \) be a basic feasible solution of \eqref{eq:opt_policy_LP}, the optimal feasible policy \( \pi^*_t \) with a support of at most \( m+1 \) elements is given by 
    \begin{equation}\label{eq06}
        \Prob_{x\sim\pi^*_t}(x=z)=\begin{cases}
            \beta^*_i & z=z_i \\
            0 & \text{otherwise}
        \end{cases}
    \end{equation}
    where \( z_i \) is the \( i \)th column of \( Z \).
    This completes the proof.
\end{proof}
The proof of Theorem~\ref{thm_opt_policy} provides a straightforward algorithm to compute the optimal feasible policy \(\pi^*_t\) given \(\theta_*\) and \(\Gamma_*\). Algorithm~\ref{alg1} provides the steps to achieve this goal.

\begin{algorithm}[H]
    \caption{Computation of the optimal feasible policy}\label{alg1}
    \begin{algorithmic}[1]
        \Statex \textbf{Input:} $\theta_*\in\reals^d, \Gamma_*\in\reals^{m\times d}, \tau\in\reals^d_+$
        \State\label{alg1l1} Solve \eqref{eq04} and find $z^*\in\Co\left(\calD_t\right)$
        \State\label{alg1l2} Find $Z=\begin{bmatrix} z_1 \cdots z_{d+1} \end{bmatrix} \in \reals^{d\times (d+1)}$ and $\alpha\in\reals^d_+$ such that $z^*=Z\alpha$ and $\ones^\top \alpha = 1$
        \State\label{alg1l3} Find $\beta^*$, a basic feasible solution of \eqref{eq:opt_policy_LP}
        \State\label{alg1l4} \Return $\pi^*_t$ according to \eqref{eq06}
        % \State \Return 
    \end{algorithmic}
\end{algorithm}

With the decision set described in \eqref{eq:union_of_convex}, lines \ref{alg1l1} and \ref{alg1l2} of the algorithm can be implemented simultaneously using the result of Lemma \ref{lemma:union_of_convex}. According to Lemma \ref{lemma:union_of_convex}, this can be done by solving the convex optimization problem 
\begin{align}\label{eq07}
    \begin{split}
        \minimize \quad & \theta^\top z \\
        \st \quad & \Gamma z \leq \tau \\
        & \alpha_i f_{ij}(x_i/\alpha_i)\leq 0, \quad i\in[k], j\in[k_i]\\
        & \ones^\top \alpha = 1 \\
        & \alpha \geq 0 \\
        & z = x_1 + \cdots + x_k
    \end{split}
\end{align}
and finding the optimal $z^* = \sum_{i=1}^{k} \alpha_i z_i$, where $z_i=x_i/\alpha_i$ and $x_i, \alpha_i$ are solutions of \eqref{eq07}. Note that in this case, instead of expressing \(z^*\) in terms of at most \(d+1\) points, it is expressed in terms of \(k\) points. Based on how \(d\) and \(k\) compare, this can be a computational advantage or disadvantage. However, it does not affect the overall flow of Algorithm \ref{alg1} as all the steps can be implemented and the only difference is that \(d+1\) gets substituted by \(k\).

While we have addressed the implementation issue in line \ref{alg1l2} of Algorithm \ref{alg1} for a special case, we do not have a general computationally efficient method to implement it without further knowledge of the set \(\calD_t\) and the way it is being expressed.

Finally, line~\ref{alg1l3} of Algorithm~\ref{alg1} can be implemented using the Simplex method, and line~\ref{alg1l4} is constructed based on the output of line~\ref{alg1l3}, which concludes our full algorithmic pipeline to calculate the optimal feasible policy \(\pi^*_t\).

\subsection{Computationally-tractable algorithms with performance guarantees for linear bandits with linear constraints}\label{sec:l1_oplb}

In the literature, there are numerous formulations of linearly-constrained linear bandits \cite{amani2019linear, moradipari2021safe, pacchiano2021stochastic, moradipari2020stage,kazerouni2017conservative,khezeli2020safe}. Many associated algorithms \cite{kazerouni2017conservative, moradipari2020stage, moradipari2021safe, pacchiano2021stochastic} follow similar strategies. Specifically, they establish confidence regions for both reward and cost parameters. These algorithms strive to optimistically maximize the reward, while taking a pessimistic stance in controlling the cost. This means they account for the worst-case scenario that the cost parameter corresponds to the least favorable value within the confidence region.

In this section, we delve into the Optimistic-Pessimistic Linear Bandit (OPLB) Algorithm introduced by \cite{pmlr-v130-pacchiano21a}, which serves as our foundational algorithm. We elucidate its workings, identify computational barriers, and tackle these challenges by introducing computationally-tractable algorithms backed by performance guarantees. It is worth noting that, although our solutions are tailored to a specific formulation of the linearly constrained linear bandit problem, they can be readily extended to other formulations, given that they all encounter the same computational challenge. 

Consider a linear bandit with linear constraints as described in \ref{sec:prob_formulation}. For simplicity we assume \( m=1 \). Consequently, the constraint matrix \( \Gamma_* \in \mathbb{R}^{m \times d} \) simplifies to a row vector, which we denote by \( \mu_*^\top \in \mathbb{R}^d \). This implies that only one linear constraint, \( \mu_*^\top x \leq \tau \), is present. Extending this to the general case with \( m \) constraints is straightforward.

At each round \( t\in[T] \), given the past actions \( \{x_i\}_{i=1}^{t-1} \), observed rewards \( \{r_i\}_{i=1}^{t-1} \), and cost signals \( \{c_i\}_{i=1}^{t-1} \), we construct the Gram matrix 
\begin{equation}\label{gram_matrix}
    \Sigma_t = \lambda I + \sum_{i=1}^{t-1}x_i x_i^\top.
\end{equation}
Then we compute the \( \ell_2 \)-regularized least squares estimates of \( \theta_* \) and \( \mu_* \) using the regularization parameter \( \lambda \). These are given by 
\begin{equation}\label{eq:rls}
    \thetahat_t = \Sigma_t^{-1} \sum_{i=1}^{t-1}r_ix_i, \quad \quad \muhat_t = \Sigma_t^{-1} \sum_{i=1}^{t-1}c_ix_i.
\end{equation}

As suggested by OPLB, we construct the confidence sets 
\begin{align}\label{conf_sets}
\begin{split}
    \calC_{t,\ell_2}^\theta&=\left\{\theta\in\reals^d:\left\|\theta-\thetahat_t\right\|_{\Sigma_t,2}\leq \rho\beta_t\right\},\\
    \calC_{t,\ell_2}^\mu&=\left\{\mu\in\reals^d:\left\|\mu-\muhat_t\right\|_{\Sigma_t,2}\leq \beta_t\right\},
\end{split}
\end{align}
where \( \rho=1+\frac{2}{\tau-c_0} \), \( \beta_t=R\sqrt{d\log \frac{1+(t-1)L^2/\lambda}{\delta}}+\sqrt{\lambda}S \), and \(\|.\|_{\Sigma_t, 2}\) is defined in Sectin \ref{sec:prelim}.

According to the principal theorem presented in \cite{NIPS2011_e1d5be1c}, there is a probability of at least \(1-\delta\) that the unidentified parameters \(\theta_*\) and \(\mu_*\) are contained within the sets \(\calC_{t,\ell_2}^\theta\) and \(\calC_{t,\ell_2}^\mu\), respectively.

The final step of OPLB is to solve the problem 
\begin{align}\label{eq:oplb_org}
    \begin{split}
        \maximize_{\pi\in\Delta_{\calD_t}, \theta\in\reals^d}\quad &\Expect_{x\sim\pi}\left(\theta^\top x\right) \\
        \st \quad& \theta\in \calC_{t,\ell_2}^\theta\\
         & \pi\in\Pi_t,
    \end{split}
\end{align}
where 
\begin{equation}\label{eq:safe_policies}
    \Pi_t=\{\pi\in\Delta_{\calD_t}:\Expect_{x\sim\pi}\left(\mu^\top x\right) \leq \tau,\: \forall \mu\in\calC^\mu_{t,\ell_2}\}
\end{equation}
is the pessimistic set of safe policies. 

\begin{proposition}\label{prop:oplb_prob}
    The optimization problem \eqref{eq:oplb_org} is equivalent to 
    \begin{align}\label{eq:oplb_prob}
    \begin{split}
        \maximize_{z\in\reals^d}\quad &\rho\beta_t \sqrt{z^\top \Sigma_t z} + \thetahat_t^\top z \\
        \st \quad& \beta_t \sqrt{z^\top \Sigma_t z} + \muhat_t^\top z\leq \tau\\
         & z\in\Co(\calD_t).
    \end{split}
\end{align}
\end{proposition}
\begin{proof}
    First, we define \( z = \mathbb{E}_{x \sim \pi}(x) \). Instead of tackling an optimization problem over a set of probability distributions, we aim to find the expected value. This step needs the condition \( z \in \Co(\mathcal{D}_t) \). This reasoning follows the same lines as the proof of Theorem \ref{thm_opt_policy}. The remainder of the proof, which explains the specific forms of the objective function and the constraint, directly stems from Proposition 1 in \cite{pmlr-v130-pacchiano21a}.
\end{proof}
Once equation \eqref{eq:oplb_prob} is solved and the optimal solution \(z^* = \sum_{i=1}^{d+1} \alpha_i z_i\) is identified as a convex combination of elements from \(\mathcal{D}_t\), the optimal feasible policy \( \pi^*_t \) is expressed by
    \begin{equation}\label{eq:opt_policy_oplb}
        \Prob_{x\sim\pi^*_t}(x=z) = 
        \begin{cases}
            \alpha_i & \text{if } z=z_i \\
            0 & \text{otherwise}.
        \end{cases}
    \end{equation}

The following theorem, a central result from \cite{pmlr-v130-pacchiano21a}, offers a regret bound on the algorithm's performance.
\begin{theorem}[Theorem 2 of \cite{pmlr-v130-pacchiano21a}]\label{thm:oplb_regret}
    Assuming the conditions presented in the problem formulation of Section \ref{sec:prob_formulation} are satisfied, the regret of OPLB, with a probability greater than \(1-2\delta\), is bounded by
    \begin{align}\label{eq:oplb_regret_bound}
    \begin{split}
        \mathcal{R}_\Pi\left(\theta, T\right) &\leq \frac{2L(\rho+1)\beta_T}{\sqrt{\lambda}}\sqrt{2T\log\left(1/\delta\right)} \\ 
        &+ (\rho+1)\beta_T\sqrt{2Td\log\left(1+\frac{TL^2}{\lambda}\right)}.
    \end{split}
    \end{align}
\end{theorem}
While the outlined approach offers a comprehensive pipeline to tackle the constrained bandit problem, a primary obstacle arises from the computational complexity of solving the main optimization problem \eqref{eq:oplb_org} or its equivalent \eqref{eq:oplb_prob}. As noted in \cite{Dani2008StochasticLO}, the unconstrained variant of this problem, with a decision set that is represented as a polytope defined by the intersection of halfspaces, is NP-hard. This implies that searching for a universally applicable computational technique, irrespective of the decision set's nature, may be futile. Furthermore, as elaborated in Section \ref{sec:opt_policy}, optimizing over probability distributions (or equivalently, with the constraint \(z\in\Co\left(\calD_t\right)\)) introduces its own set of challenges.

To navigate the first challenge, we propose a modified OPLB that, while computationally feasible, yields a more relaxed regret bound. This modification ensures a universally efficient algorithm. Later in Section \ref{sec:linear_relaxation}, we present an alternative technique for addressing the original problem \eqref{eq:oplb_prob}, which is suitable for specific cases but not universally. To tackle the second challenge, analogous to Section \ref{sec:opt_policy}, we utilize the technique introduced in Lemma \ref{lemma:union_of_convex}.

To make the OPLB more computationally efficient, we modify the confidence sets. Instead of using the confidence set \( C^\theta_{t,\ell_2} \) presented in \eqref{conf_sets}, we switch to a confidence set using the \( \ell_1 \) norm and an adjusted radius. Specifically, we define the confidence set as
\begin{equation}\label{eq:conf_set_l1}
    \calC_{t,\ell_1}^\theta = \left\{ \theta \in \reals^d : \left\| \theta-\thetahat_t \right\|_{\Sigma_t,1} \leq \rho\sqrt{d}\beta_t \right\},
\end{equation}
where \( \rho \) and \( \beta_t \) retain their previous definitions and \(\|.\|_{\Sigma,1}\) is detailed in Section \ref{sec:prelim}.
Note that, as will be shown in a subsequent lemma, an \( \ell_1 \) confidence set for \( \mu_* \) is unnecessary. Instead, we can continue using the \( \calC_{t,\ell_2}^\mu \) as previously defined.
\begin{lemma}\label{lemma:l1_conf_bound}
    For any \( t \in [T] \) and any \( \delta > 0 \), the following holds:
    \begin{align}\label{eq:l1_conf_bound}
        \Prob\left(\theta_*\in\calC_{t,\ell_1}^\theta\right) &\geq 1-\delta.
    \end{align}
\end{lemma}
\begin{proof}
    For any vector \( x \in \reals^d \), we have that \( \|x\|_1 \leq \sqrt{d} \|x\|_2 \). This yields 
    \[ \left\|\Sigma^{1/2}\left(\theta - \thetahat\right)\right\|_1 \leq \sqrt{d} \left\|\Sigma^{1/2}\left(\theta - \thetahat\right)\right\|_2. \]
    Given that the right-hand side is bounded by \( \sqrt{d} \rho \beta_t \) for any \( \theta \in \calC^\theta_{t, \ell_2} \), it follows that
    \( \calC^\theta_{t, \ell_2} \subseteq \calC^\theta_{t, \ell_1} \).
    By the main theorem of \cite{NIPS2011_e1d5be1c}, we know that \( \theta_* \in \calC^\theta_{t, \ell_2} \) with a probability of at least \( 1-\delta \), which concludes the proof.
\end{proof}
In the modified version of OPLB that incorporates the $\ell_1$ confidence region, we address a new optimization problem given by
\begin{align}\label{eq:l1_OPLB_org}
    \begin{split}
        \maximize_{z\in\reals^d, \theta\in\reals^d}\quad &\theta^\top z \\
        \st \quad& \theta\in \calC_{t, \ell_1}^\theta\\
         & z\in S_t\\
         & z\in\Co\left(\calD_t\right),
    \end{split}
\end{align}
where \(S_t=\{z\in\reals^d:\mu^\top z \leq \tau,\: \forall \mu\in\calC^\mu_{t,\ell_2}\}.\) Once this problem is solved and the optimal solution \(z^* = \sum_{i=1}^{d+1} \alpha_i z_i\) is identified as a convex combination of elements from \(\mathcal{D}_t\), the optimal feasible policy \( \pi^*_t \) is given by \eqref{eq:opt_policy_oplb}.
\begin{proposition}\label{prop:l1_efficiency}
    The optimization problem \eqref{eq:l1_OPLB_org} can be decomposed and solved by addressing $2d$ individual convex optimization problems.
\end{proposition}
\begin{proof}
    We can express \eqref{eq:l1_OPLB_org} in the following format:
    \begin{align}\label{eq:l1_OPLB_theta_prob}
    \begin{split}
        \maximize_{\theta\in\reals^d}\quad &f(\theta) \\
        \st \quad& \theta\in \calC_{t, \ell_1}^\theta,\\
    \end{split}
    \end{align}
where the function $f$ is defined as:
\begin{align}\label{eq:l1_OPLB_z_prob}
    \begin{split}
        f(\theta)=\quad\max_{z\in\reals^d}\quad &\theta^\top z \\
        \mathrm{s.t.} \quad&  z\in S_t\\
        & z\in\Co\left(\calD_t\right).
    \end{split}
\end{align}
Given that $f$ is convex and the region $\calC_{t,\ell_1}^\theta$ forms a polytope in $\reals^d$, our task in \eqref{eq:l1_OPLB_theta_prob} is to maximize this convex function over the polytope. Recognizing that solutions to such problems occur at the vertices of the polytope, we realize that to solve \eqref{eq:l1_OPLB_theta_prob}, it suffices to evaluate $f$ at the $2d$ vertices. Each evaluation corresponds to solving a convex optimization problem as shown in \eqref{eq:l1_OPLB_z_prob}, which completes the proof.
\end{proof}
Proposition \ref{prop:l1_efficiency} demonstrates that the modified OPLB can be efficiently solved. The subsequent step is to ascertain a guarantee for the regret bound. The theorem below provides this guarantee.
\begin{theorem}[Modified OPLB regret bound]\label{thm:l1_oplb_regret}
    Given that the conditions outlined in Section \ref{sec:prob_formulation} are met, the regret of the modified OPLB employing the $\ell_1$ confidence region for the reward parameter $\theta_*$, with a probability exceeding \(1-2\delta\), can be upper-bounded as
    \begin{align}\label{eq:l1_oplb_regret_bound}
    \begin{split}
        \mathcal{R}_\Pi\left(\theta, T\right) &\leq \frac{2L(\rho+1)\beta_T}{\sqrt{\lambda}}\sqrt{2Td\log\left(1/\delta\right)} \\ 
        &+ (\rho+1)\beta_Td\sqrt{2T\log\left(1+\frac{TL^2}{\lambda}\right)}.
    \end{split}
    \end{align}
\end{theorem}
\begin{proof}
    By examining the proof of Theorem \ref{thm:oplb_regret}, it becomes apparent that the regret bound depends on the confidence region radius of the reward parameter $\theta_*$, namely $\rho\beta_T$, without specifically relying on the value of $\rho\beta_T$. Further inspection reveals that the confidence region radius of the cost parameter $\mu_*$ has no bearing on the bound. In the modified OPLB approach, the initial radius is scaled by a factor of $\sqrt{d}$, while the latter remains unchanged. Hence, in the expression \eqref{eq:oplb_regret_bound}, substituting $\beta_t$ with $\sqrt{d}\beta_t$ results in the updated bound presented in \eqref{eq:l1_oplb_regret_bound}, which completes the proof.
\end{proof}
With Theorem \ref{thm:l1_oplb_regret}, we now possess a comprehensive framework for tackling the constrained bandit problem using algorithms that are computationally efficient. It's important to highlight that a key step in this process is the evaluation of the function \(f\) as defined in \eqref{eq:l1_OPLB_z_prob}. Although this is a convex optimization problem, one cannot overlook that its two constraints, in their most general form, may introduce complications unless they are further simplified.

The primary constraint, \(z\in S_t\), can be replaced by the more direct constraint \(\beta_t \sqrt{z^\top \Sigma_t z} + \muhat^\top z\leq \tau\), following the guidelines of Proposition \ref{prop:oplb_prob}. The latter constraint, \(z\in\Co\left(\calD_t\right)\), while complicated in general, is navigated for the class of decision sets studied in this work using the technique introduced in Lemma \ref{lemma:union_of_convex}. Consequently, the task of evaluating the function \(f\) from \eqref{eq:l1_OPLB_z_prob} simplifies to solving
\begin{equation}\label{eq:l1_OPLB_z_prob_final}
    \begin{split}
        f(\theta) = \max_{z\in\reals^d} &\quad  \theta^\top z \\
        \mathrm{s.t.} \quad&  \beta_t \sqrt{z^\top \Sigma_t z} + \muhat_t^\top z \leq \tau\\
        & \alpha_i f_{ij}(x_i/\alpha_i) \leq 0, \quad i\in[k], j\in[k_i]\\
        & \ones^\top \alpha = 1, \quad \alpha \geq 0, \\
        % & \alpha \geq 0 \\
        & z = x_1 + \cdots + x_k.
    \end{split}
\end{equation}
The procedure is concisely summarized in Algorithm \ref{alg:l1_oplb}.

\begin{algorithm}[H]
    \caption{Modified OPLB}\label{alg:l1_oplb}
    \begin{algorithmic}[1]
        \Statex \textbf{Input:} \( T\in\N \), \( \delta\in\reals_+ \), \( \gamma\in\reals_+ \), \( \tau\in\reals \)
        \For{\( t = 1 \) \textbf{to} \( T \)}
            \State Observe \( r_t, c_t \) and compute \( \thetahat_t \) and \( \muhat_t \) using \eqref{eq:rls}
            \State \(z^*\gets\) Solve \eqref{eq:l1_OPLB_org} using Proposition \ref{prop:l1_efficiency} and \eqref{eq:l1_OPLB_z_prob_final}
            \State Construct \( \pi_t \) using \(z^*\) according to \eqref{eq:opt_policy_oplb} 
            \State Play the action \( x_t \sim \pi_t \)
        \EndFor
    \end{algorithmic}
\end{algorithm}

\subsection{The upper bound maximization method}\label{sec:linear_relaxation}
Recall \eqref{eq:oplb_prob} which presents the original problem with \(\ell_2\) confidence sets that we initially sought to solve. Since solving this problem is challenging, our first approach was to present an \(\ell_1\) relaxation to this problem, as discussed in Section \ref{sec:l1_oplb}. In this section we introduce a problem-dependent method that has the potential to exactly solve \eqref{eq:oplb_prob}. The following theorem provides the tools that we need for this method.
\begin{theorem}\label{thm:upper_bound_relaxation}
Let \( g_1, g_2 : \reals^d \to \reals \) be arbitrary functions and let \( C \subseteq \reals^d \) be an arbitrary set.
    Consider the optimization problems 
    \begin{align}\label{eq:upper_bound_relaxation_prob1}
        \begin{split}
            \maximize_{z \in \reals^d} \quad & g_1(z) \\
            \st \quad & g_1(z) \leq g_2(z) \\
            & z \in C
        \end{split}
    \end{align}
    and
    \begin{align}\label{eq:upper_bound_relaxation_prob2}
        \begin{split}
            \maximize_{z \in \reals^d} \quad & g_2(z) \\
            \st \quad & g_1(z) \leq g_2(z) \\
            & z \in C.
        \end{split}
    \end{align}
    If \( z^* \) is an optimal solution for \eqref{eq:upper_bound_relaxation_prob2} and \( g_1(z^*) = g_2(z^*) \), then \( z^* \) is also an optimal solution for \eqref{eq:upper_bound_relaxation_prob1}.
\end{theorem}
\begin{proof}
    Suppose \( \tilde{z} \) is an optimal solution for \eqref{eq:upper_bound_relaxation_prob1} and \( z^* \) is not. Then, \( g_1(\tilde{z}) > g_1(z^*) = g_2(z^*) \geq g_2(\tilde{z}) \). The first inequality stems from the optimality of \( \tilde{z} \) and the non-optimality of \( z^* \) for \eqref{eq:upper_bound_relaxation_prob1}, the equality follows directly from the assumption of the theorem, and the last inequality arises because \( z^* \) maximizes \( g_2 \). This leads to \( g_1(\tilde{z}) > g_2(\tilde{z}) \), a violation of the constraint \( g_1(z) \leq g_2(z) \), thus a contradiction. This concludes that \( z^* \) is an optimal solution for \eqref{eq:upper_bound_relaxation_prob1}.
\end{proof}
This theorem allows us to solve \eqref{eq:upper_bound_relaxation_prob2} instead of \eqref{eq:upper_bound_relaxation_prob1}. If the condition \( g_1(z^*) = g_2(z^*) \) holds, then we have an optimal solution for \eqref{eq:upper_bound_relaxation_prob1} as well. This may be quite useful if \eqref{eq:upper_bound_relaxation_prob2} is more tractable than \eqref{eq:upper_bound_relaxation_prob1}.

We now apply the result of Theorem~\ref{thm:upper_bound_relaxation} to \eqref{eq:oplb_prob}. 
For clarity, we restate this problem as follows:
\begin{align}\label{eq:oplb_prob2}
    \begin{split}
        \maximize_{z \in \reals^d} \quad & \rho \beta_t \sqrt{z^\top \Sigma_t z} + \hat{\theta}_t^\top z \\
        \st \quad & \beta_t \sqrt{z^\top \Sigma_t z} + \hat{\mu}_t^\top z \leq \tau \\
         & z \in \Co(\mathcal{D}_t).
    \end{split}
\end{align}
Setting \(g_1(z) = \rho \beta_t \sqrt{z^\top \Sigma_t z} + \hat{\theta}_t^\top z\), \(g_2(z) = \rho \tau + (\hat{\theta}_t - \rho \hat{\mu}_t)^\top z\), and \(C = \Co(\mathcal{D}_t)\), \eqref{eq:oplb_prob2} becomes a particular instance of \eqref{eq:upper_bound_relaxation_prob1}. Consequently, the counterpart of \eqref{eq:upper_bound_relaxation_prob2} in our setting is
\begin{align}\label{eq:oplb_prob_linear}
    \begin{split}
        \maximize_{z \in \reals^d} \quad & \rho \tau + \left(\hat{\theta}_t - \rho \hat{\mu}_t\right)^\top z \\
        \st \quad & \beta_t \sqrt{z^\top \Sigma_t z} + \hat{\mu}_t^\top z \leq \tau \\
         & z \in \Co(\mathcal{D}_t),
    \end{split}
\end{align}
which is a convex optimization problem. This provides a potentially more efficient approach to solve the original OPLB problem with \(\ell_2\) confidence sets and yield exact solutions. It involves solving \eqref{eq:oplb_prob_linear}, a convex optimization problem amenable to efficient computation. Upon solving this problem, one must check whether the first constraint is active. If it is, then the obtained solution also solves \eqref{eq:oplb_prob2}. If not, the process shifts back to addressing the \(\ell_1\) version of the problem, as outlined in \eqref{eq:l1_OPLB_org} or \eqref{eq:l1_OPLB_z_prob_final}. To handle the second constraint in \eqref{eq:oplb_prob_linear}, we utilize the technique proposed in Lemma~\ref{lemma:union_of_convex}. Algorithm~\ref{alg:linear} summarizes the entire methodology. We refer to this technique as the \textit{Upper Bound Maximization (UBM)} method, as it entails maximizing an upper bound on the objective function rather than the objective function itself.

\begin{algorithm}[H]
    \caption{Enhanced OPLB with UBM}\label{alg:linear}
    \begin{algorithmic}[1]
        \Statex \textbf{Input:} \( T \in \mathbb{N} \), \( \delta \in \mathbb{R}_+ \), \( \gamma \in \mathbb{R}_+ \), \( \tau \in \mathbb{R} \)
        \For{\( t = 1 \) \textbf{to} \( T \)}
            \State Observe \( r_t, c_t \), and compute \( \hat{\theta}_t \), \( \hat{\mu}_t \) using \eqref{eq:rls}
            \State \(z^* \gets\) Solve \eqref{eq:oplb_prob_linear}
            \If{\(\beta_t \sqrt{{z^*}^\top \Sigma_t {z^*}} + \hat{\mu}_t^\top {z^*} < \tau\)}
                \State \(z^* \gets\) Solve \eqref{eq:l1_OPLB_org} using Proposition~\ref{prop:l1_efficiency} and \eqref{eq:l1_OPLB_z_prob_final}
            \EndIf
            \State Construct \( \pi_t \) using \(z^*\) according to \eqref{eq:opt_policy_oplb}
            \State Play the action \( x_t \sim \pi_t \)
        \EndFor
    \end{algorithmic}
\end{algorithm}

Each iteration of Algorithm~\ref{alg:linear} involves solving either the \(\ell_2\) or the \(\ell_1\) confidence set problem. Thus, the ultimate regret bound will be no worse than that provided by Theorem~\ref{thm:l1_oplb_regret} but may approach the bound of Theorem~\ref{thm:oplb_regret}, depending on the frequency at which the first constraint becomes active in \eqref{eq:oplb_prob_linear}.

\textbf{Example.} Figure~\ref{fig:lin_relax_example} illustrates a one-dimensional example of our setup. The decision set \(\calD_t\) is defined such that \(\Co\left(\calD_t\right)=\{z : |z| \leq 3\}\). Two distinct upper bound functions, \(g^{(1)}_{2}(z)\) and \(g^{(2)}_{2}(z)\), are introduced, each corresponding to a different value of \(\muhat_t\). The set \(S\) represents the points where the safety constraint \(g_1(z) \leq g_2(z)\), as described in \eqref{eq:oplb_prob2} and \eqref{eq:oplb_prob_linear}, is satisfied.

The implications of Theorem~\ref{thm:upper_bound_relaxation} are observable in Figure~\ref{fig:lin_relax_example}, where the conditions under which UBM is effective become apparent. Specifically, when the upper bound is described by \(g^{(1)}_{2}(z)\), maximizing this function also optimizes the original objective \(g_{1}(z)\), with the constraint \(g_1(z) \leq g_2(z)\) becoming active at the optimum. Conversely, when the upper bound is \(g^{(2)}_{2}(z)\), UBM does not lead to an optimal solution, as maximizing \(g^{(2)}_{2}(z)\) does not make the constraint active, rendering the approach ineffective in this case.

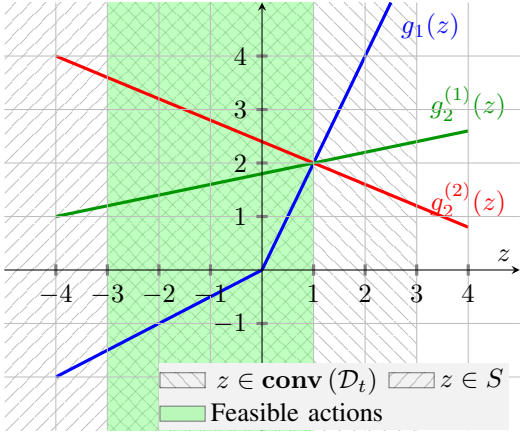
\begin{figure}
    \centering
    \begin{tikzpicture}[scale=1]

    \begin{axis}[
      axis lines=middle,
      grid=major,
      xmin=-5,
      xmax=5,
      ymin=-3,
      ymax=5,
      xlabel=$z$,
      xtick={-4,-3,...,4},
      ytick={-2,-1,...,4},
      tick style={very thick},
      legend style={
      at={(rel axis cs:.3,0)},
      legend columns=2,
      axis on top,
      anchor=south west,draw=none,inner sep=0pt,fill=gray!10}
    ]
    \fill[pattern={Lines[angle=-45,distance=4pt]}, opacity=.5] (axis cs:-3,-3) -- (axis cs:3,-3) -- (axis cs:3,10) -- (axis cs:-3,10) -- cycle;
    \fill[pattern={Lines[angle=45,distance=4pt]}, opacity=.5] (axis cs:-5,-3) -- (axis cs:1,-3) -- (axis cs:1,10) -- (axis cs:-5,10) -- cycle;
    \fill[green!50, opacity=.5] (axis cs:-3,-3) -- (axis cs:1,-3) -- (axis cs:1,10) -- (axis cs:-3,10) -- cycle;

    \addlegendimage{area legend,pattern={Lines[angle=-45,distance=4pt]},pattern color=black,opacity=.5}
    \addlegendentry{$z\in\Co\left(\calD_t\right)$}
    \addlegendimage{area legend,pattern={Lines[angle=45,distance=4pt]},pattern color=black,opacity=.5}
    \addlegendentry{$z\in S$}
    \addlegendimage{area legend,fill=green!50, opacity=.5}
    \addlegendentry{Feasible actions}
    
    \addplot[blue,very thick,samples=100,domain=-4:0] {.5*x};
    \addplot[blue,very thick,samples=100,domain=0:2.5] {2*x}node[below right] {$g_{1}(z)$};
    \addplot[red,very thick,samples=100,domain=-4:4] {-.4*x+2.4} node[above] {$g^{(2)}_{2}(z)$};
    \addplot[darkgreen,very thick,samples=100,domain=-4:4] {.2*x+1.8} node[above] {$g^{(1)}_{2}(z)$};

    % \legend{green}
\end{axis}

% \draw[<->] (1.35, -.5) -- (5.45, -.5);
% \node at (3.4, -.5) [fill=white] {$\Co\left(\calD_t\right)$};
% \draw[->] (0, -1) -- (4.2, -1);
% \node at (.15, -1) [fill=white] {$\cdots$};
% \node at (2.2, -1) [fill=white] {${\left\{z\:\colon h_t(z)\leq 0\right\}}$};

\end{tikzpicture}

    \caption{Example of the upper bound maximization test}
    \label{fig:lin_relax_example}
\end{figure}

\section{Experiments}\label{sec:experiments}
In this section, we present empirical evaluations of the proposed algorithms through two distinct experiments.
\subsection{Enhanced OPLB policy evaluation with non-convex decision sets}
\begin{figure}
    \centering
    \includegraphics[width=.4\textwidth]{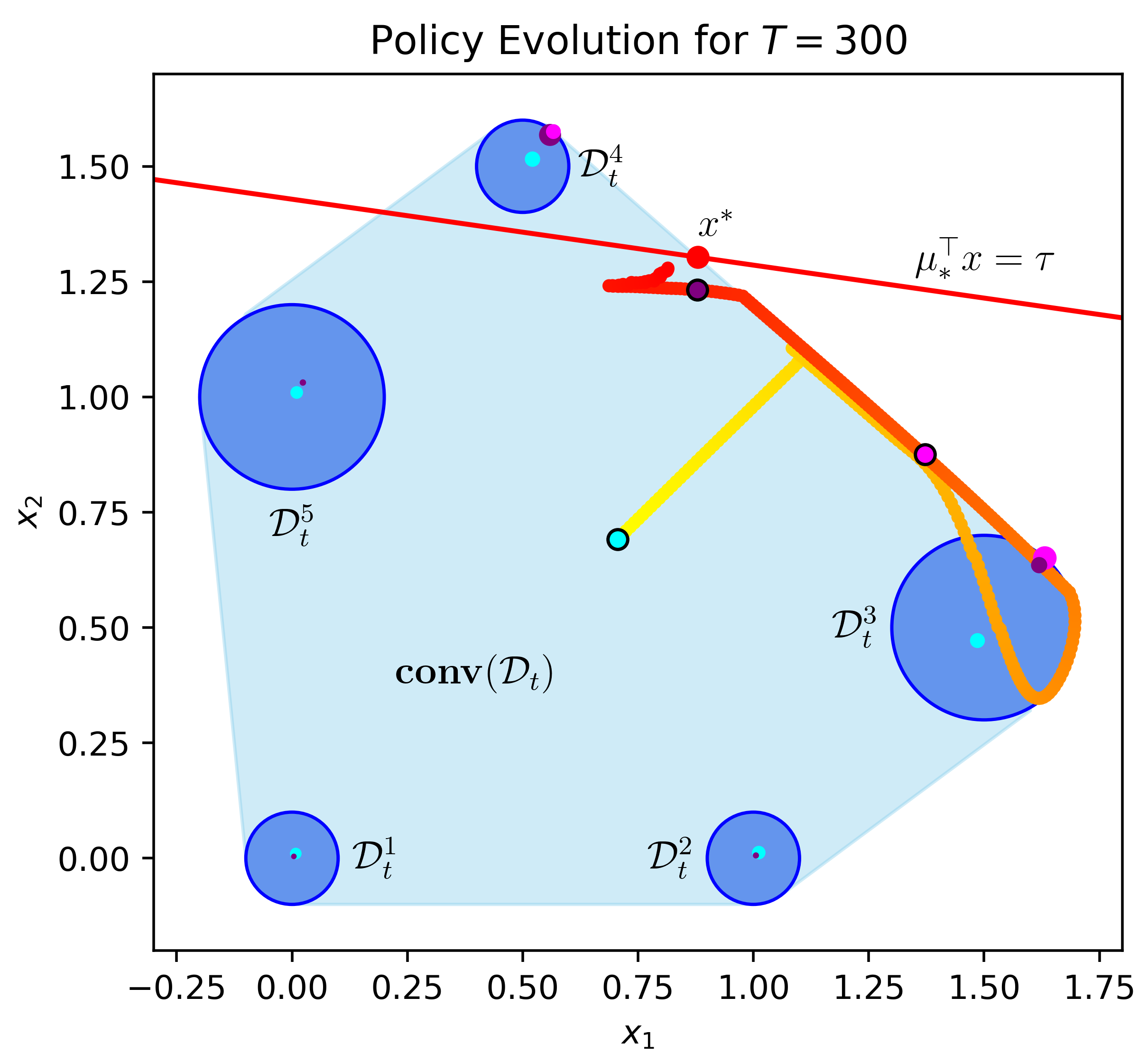}
    \caption{Mean policy trajectory with \(\calD_t\) as a union of convex sets}
    \label{fig:union_of_circles}
\end{figure}

The first experiment considers a two-dimensional scenario with a non-convex decision set represented by a union of five disks in \( \mathbb{R}^2 \), all subject to a single linear constraint. Figure~\ref{fig:union_of_circles} illustrates the policies chosen by the algorithm at each time step over a total of \( T=300 \) rounds. The trajectory depicting the mean value of the policy is shown in Figure~\ref{fig:union_of_circles}, which transitions from yellow to red as time progresses. Notably, at three specific time steps—\( t=0 \), \( t=40 \), and \( t=100 \)— the mean policy values are highlighted in cyan, magenta, and purple respectively, each delineated with a black border. Corresponding to each of these mean values, five points are plotted, representing the five potential actions, one of which is to be randomly selected according to a specific probability for the policy to be effective. The radius of each point is proportional to its probability weight in the policy's construction, with all weights summing up to one. Furthermore, the constraint boundary, defined by \( x^\top\mu = \tau \), is represented as a line within the figure, and the mean value of the optimal policy is denoted as \(x^*\).

Observations from the figure reveal that initially, the trajectory of the points moves along the boundary of the convex hull of the decision set and away from the optimal policy. However, as time progresses, the trajectory redirects towards the optimal policy and ultimately converges to the optimal solution. Furthermore, the mean value of the policy always remains within the safe region, indicating that the pessimism in action selection has been effective, ensuring that the algorithm does not violate the safety constraint at any point.

\subsection{Cumulative regret comparison}

\begin{figure*}
\centering
\includegraphics[width=.33\textwidth]{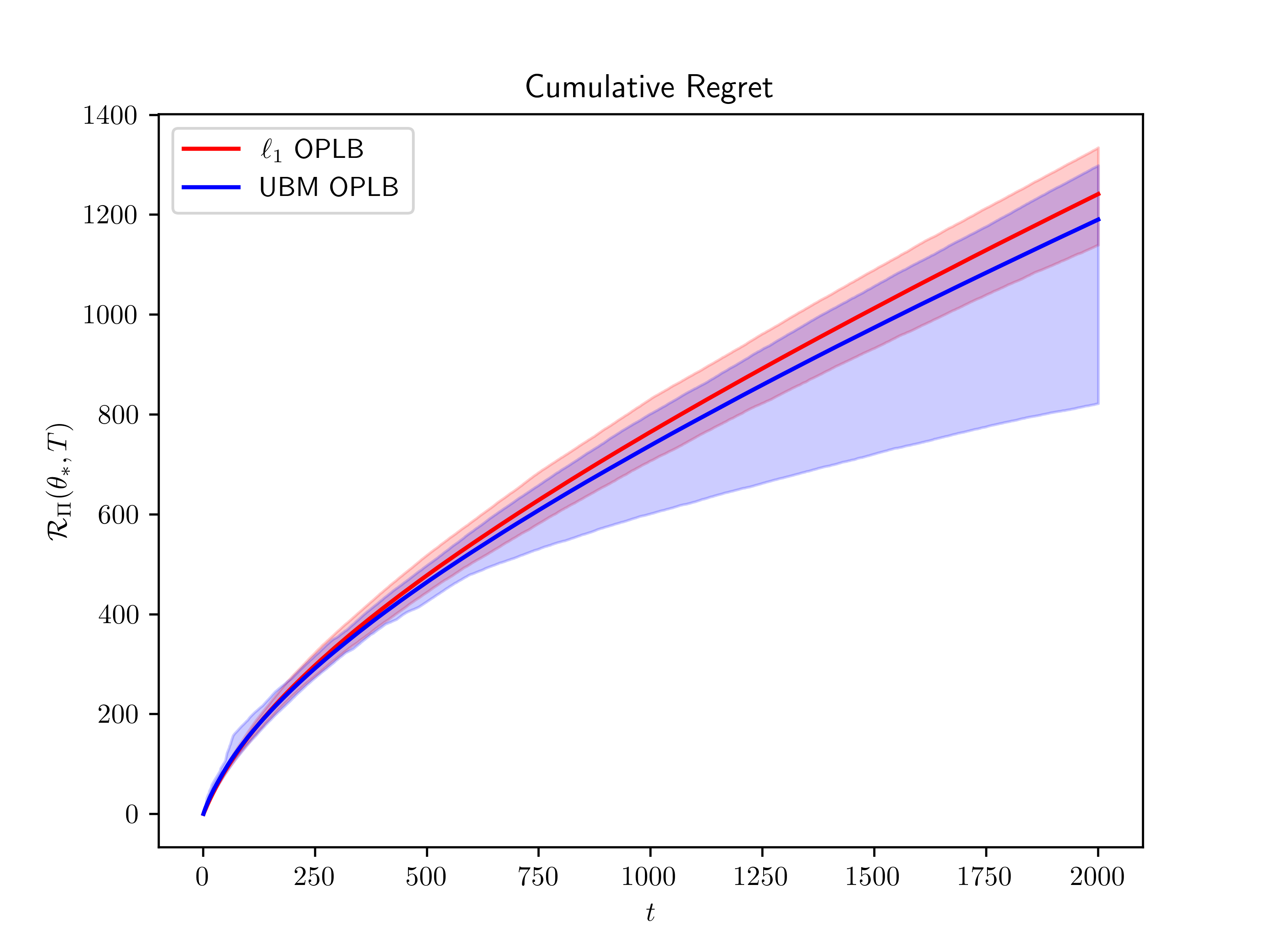}
 \includegraphics[width=.33\textwidth]{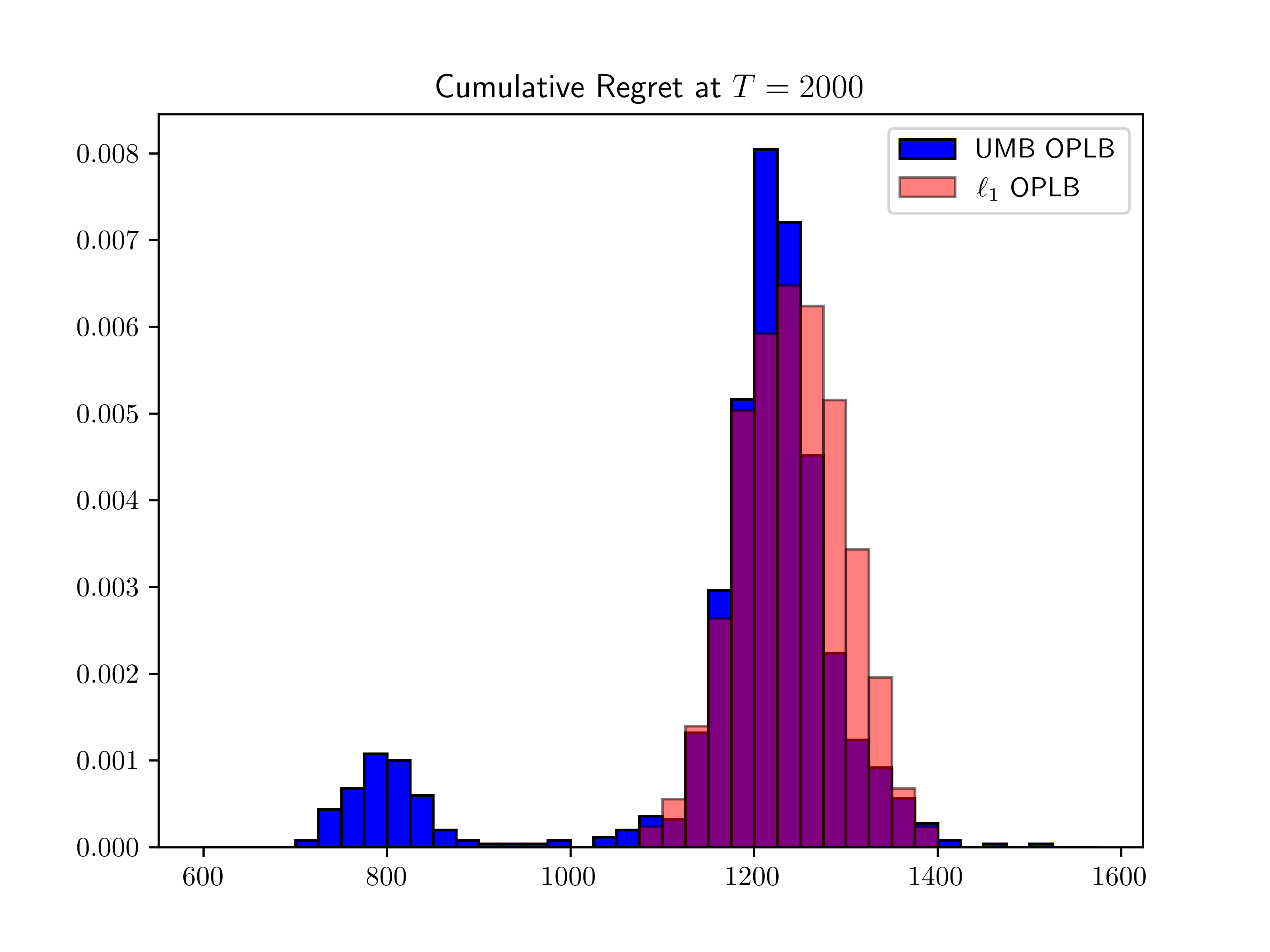}
  \includegraphics[width=.23\textwidth]{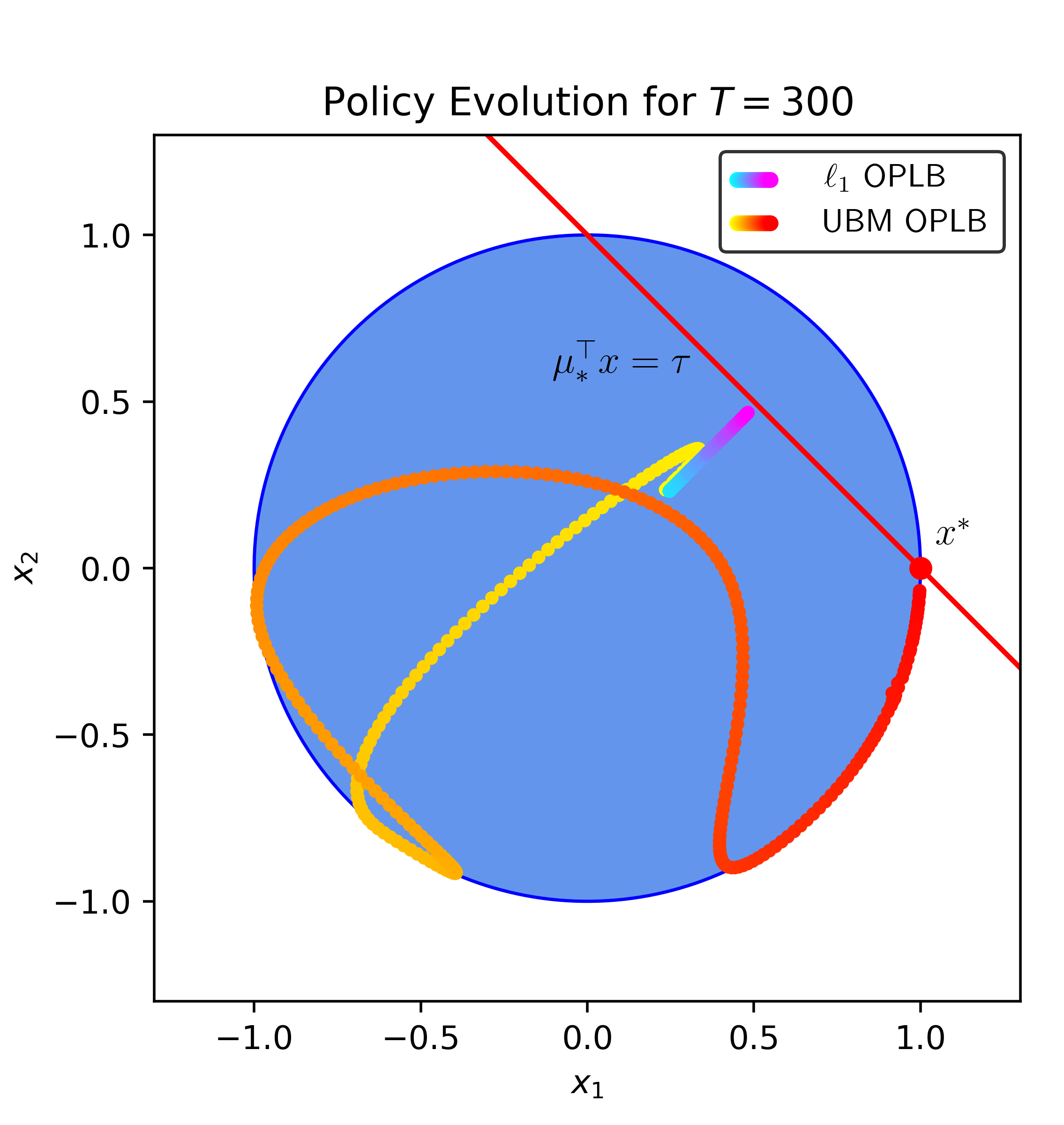}
  \caption{Left: Cumulative regret of \(\ell_1\) OPLB vs. OPLB with UBM. Middle: Histogram of cumulative regret at \(T=2000\). Right: Mean policy trajectory of \(\ell_1\) OPLB vs. UBM OPLB}
   \label{fig:cum_reg}
\end{figure*}

% \begin{figure}
%     \centering
%     \includegraphics[width=.4\textwidth]{figs/cum_regret_hist.png}
%     \caption{Cumulative regret of \(\ell_1\) OPLB vs. Algorithm \ref{alg:linear}}
%     \label{fig:cum_reg}
% \end{figure}

In the second experiment, we compare the cumulative regrets of Algorithms~\ref{alg:l1_oplb} and~\ref{alg:linear}, namely the \( \ell_1 \) OPLB and UBM OPLB. Figure~\ref{fig:cum_reg} (left) presents the cumulative regret of both algorithms given the parameters \( \theta_* = [3, 2.5]^\top \), \( \mu_* = [0.5, 0.5]^\top \), and \( \tau=0.5 \), with the decision set being the unit disk. The results indicate a marginally better cumulative regret for UBM OPLB. This plot reveals an interesting phenomenon: asymmetric confidence bands around the UBM OPLB's regret, with a lower confidence band that is notably further below the mean compared to the upper band. 
% This suggests that while the UBM OPLB typically performs similarly to the \( \ell_1 \) OPLB, it also achieves significantly lower regret in some instances.
Further investigation into this observation is conducted by examining Figure~\ref{fig:cum_reg} (middle), which displays a histogram of the cumulative regrets for both algorithms at time \( t=2000 \) over \( N=1000 \) simulations. The histogram suggests that, although UBM OPLB's performance is largely in line with that of \( \ell_1 \) OPLB, it exhibits a secondary mode where the cumulative regret is substantially lower. This accounts for the observed lower confidence band in the first plot. In certain cases, UBM OPLB significantly outperforms \( \ell_1 \) OPLB. For a closer look at this behavior, we examine the mean policy trajectories of \( \ell_1 \) OPLB and UBM OPLB under the aforementioned superior performance. Figure~\ref{fig:cum_reg} (right) delineates these trajectories with the evolution from yellow to red and cyan to magenta, respectively, for a span of \( T=300 \) steps. Clearly, \( \ell_1 \) OPLB does not approach the optimal policy as closely as UBM OPLB, resulting in greater regret, whereas UBM OPLB tends toward the optimal policy, exhibiting minimal regret. Although this phenomenon is problem-specific and not universally observed, it presents an intriguing aspect for further research.

\bibliographystyle{IEEEtran}
\bibliography{refs.bib}
\end{document}